\pdfoutput=1

\documentclass{article}
\usepackage{kr}

\usepackage{times}
\usepackage{soul}
\usepackage{url}
\usepackage[hidelinks]{hyperref}
\usepackage[utf8]{inputenc}
\usepackage[small]{caption}
\usepackage{graphicx}
\usepackage{amsmath}
\usepackage{amsthm}
\usepackage{amssymb }
\usepackage{booktabs}
\usepackage[kerning=true,tracking=true, babel=true,spacing=true]{microtype}

\newtheorem{theorem}{Theorem}
\newtheorem{proposition}[theorem]{Proposition}

\newtheorem*{observation}{Observation}

\newtheorem{definition}[theorem]{Definition}

\usepackage{csquotes}
\usepackage[ruled]{algorithm2e}
\usepackage{wasysym}
\usepackage{stmaryrd}
\usepackage{tikz}

\usetikzlibrary{decorations.markings,decorations.pathreplacing}
\usetikzlibrary{shapes.geometric,automata,positioning}
\usetikzlibrary{shadows,calc}
\usetikzlibrary{decorations,decorations.pathmorphing,decorations.pathreplacing}

\makeatletter
\newcommand{\ifLatexThree}[2]{\@ifpackageloaded{xparse}{#1}{#2}}
\newcommand{\ifAMSmath}[2]{\@ifpackageloaded{amsmath}{#1}{#2}}
\newcommand{\ifMathSCR}[2]{\@ifpackageloaded{mathrsfs}{#1}{#2}}
\newcommand{\ifMathHyperREF}[2]{\@ifpackageloaded{hyperref}{#1}{#2}}
\makeatother

\ifLatexThree{%
	\NewDocumentCommand{\headword}{s o m}{\IfBooleanTF{#1}{#3}{\textbf{#3}}\IfNoValueTF{#2}{\index{#3}}{\index{#2}}}%
}{%
	\makeatletter
	\def\@headword#1{\textbf{#1}\index{#1}}%
	\def\@@headword#1{#1\index{#1}}%
	\def\headword#1{\@ifstar\@headword{#1}\@@headword{#1}}%
	\makeatother
}

\makeatletter
\@ifundefined{naturals}{}{}
\@ifundefined{ol}{}{}
\makeatother

\makeatletter
\newcommand{\textlabelmarker}[1]{%
	\protected@edef\@currentlabel{#1}%
	\phantomsection%
}
\newcommand{\textlabel}[2]{%
	\textlabelmarker{#1}%
	#1\label{#2}%
}
\makeatother

\ifx\nmableit\undefined
\makeatletter
\ifx\centernot\undefined
\newcommand*{\centernot}{%
	\mathpalette\@centernot
}
\fi
\def\@centernot#1#2{%
	\mathrel{%
		\rlap{%
			\settowidth\dimen@{$\m@th#1{#2}$}%
			\kern.5\dimen@
			\settowidth\dimen@{$\m@th#1=$}%
			\kern-.5\dimen@
			$\m@th#1\not$%
		}%
		{#2}%
	}%
}
\makeatother
\DeclareRobustCommand\nmableitSymb{\mathrel|\mkern-.5mu\joinrel\sim} %
\newcommand{\nmableit}{\ensuremath{\mbox{$\,\nmableitSymb\,$}}} %
\fi

\makeatletter
\ifMathHyperREF{%
	\newcommand{\seqref}[1]{\hyperref[{#1}]{\textup{\tagform@split{\getrefnumber{#1}}}}}%
}{%
	\newcommand{\seqref}[1]{\textup{\tagform@split{\getrefnumber{#1}}}}%
}
\newcommand\tagform@split[1]{%
	\begingroup
	\m@th\normalfont(\ignorespaces #1\unskip\@@italiccorr)%
	\endgroup
}
\makeatother

\makeatletter
\newcommand{\leqnomode}{\tagsleft@true\let\veqno\@@leqno}
\newcommand{\reqnomode}{\tagsleft@false\let\veqno\@@eqno}
\newcommand{\pushright}[1]{\ifmeasuring@#1\else\omit\hfill$\displaystyle#1$\fi\ignorespaces}
\newcommand{\pushleft}[1]{\ifmeasuring@#1\else\omit$\displaystyle#1$\hfill\fi\ignorespaces}
\newcommand{\specialcell}[1]{\ifmeasuring@#1\else\omit$\displaystyle#1$\ignorespaces\fi}

\makeatother
\newcommand{\ksIF}{\text{if }}
\newcommand{\ksTHEN}{\text{, then }}

\newcommand{\ksAND}{\text{ and }}

\newcommand{\ksForAll}{\text{ for all }}

\newcommand{\tuple}[1]{\ensuremath{\langle{#1}\rangle}}

\newcommand{\modelsOf}[1]{\ensuremath{\llbracket #1\rrbracket}}
\newcommand{\negOf}[1]{{\ensuremath{\neg{#1}}}}

\renewcommand{\modelsOf}[1]{\ensuremath{\ksMod(#1)}}

\DeclareMathOperator{\ksBel}{Bel}
\newcommand{\beliefsOf}[1]{\ensuremath{\ksBel(#1)}}

\newcommand{\setAllES}{\ensuremath{\mathcal{E}}}

\newcommand{\propLang}{\ensuremath{\mathcal{L}}}

\ifMathSCR{%
}{%
}

\usepackage{environ,xparse,xkeyval,ifthen}
\newif\ifpostulatepresent\postulatepresentfalse

\ExplSyntaxOn
\NewEnviron{postulate}[2]%
{%
	\tl_gclear:N \g_tmpa_tl%
	\tl_gput_right:Nn \g_tmpa_tl { \ifpostulatepresent\\[\smallskipamount]\fi\global\postulatepresenttrue }%
	\tl_gput_right:Nn \g_tmpa_tl { \noindent\ignorespaces\ifthenelse{\equal{#1}{}}{}{(\textlabel{#1}{pstl:#1})} & }%
	\tl_gput_right:No \g_tmpa_tl { \BODY }%
	\tl_gput_right:Nn \g_tmpa_tl { & \hspace{1.5\medskipamount} #2  }%
	\group_insert_after:N \g_tmpa_tl%
}
\ExplSyntaxOff

\renewcommand{\modelsOf}[1]{\llbracket#1\rrbracket}

\usepackage{xparse}
\DeclareDocumentCommand{\todo}{o g}{\IfNoValueTF{#1}{\begingroup\color{magenta}TODO: #2\endgroup}{\begingroup\color{magenta}#1 #2\endgroup}}

\newcommand{\revision}{\star}

\title{Iterated Belief Change, Computationally}

\author{%
Kai Sauerwald$^1$\and
Christoph Beierle$^1$ \\
\affiliations
$^1$FernUniversität in Hagen, Knowledge-Based Systems, Germany\
\emails
\{kai.sauerwald, christoph.beierle\}@fernuni-hagen.de
}

\begin{document}
    \pagestyle{plain}

\maketitle

\begin{abstract}
Iterated Belief Change is the research area that investigates principles for the dynamics of beliefs over (possibly unlimited) many subsequent belief changes. 
In this paper, we demonstrate how iterated belief change is connected to computation. 
In particular, we show that iterative belief revision is Turing complete, even under the condition that broadly accepted principles like the Darwiche-Pearl postulates for iterated revision hold. 
\end{abstract}

\section{Introduction}
\label{sec:introduction}
The question of how an agent should change her belief according to new information is a central task for agents and AI systems.
The belief change community developed a very fruitful unique research approach to this subject. 
The approach is top-down, by considering postulates, while being at the same time agent-centric, by taking a subjective perspective on the phenomenon of belief change.

Iterated belief change considers relations between a change and changes that happen priorly or posteriorly.
In this area, changes of epistemic states are considered, which may contain \enquote{meta-information} in addition to the plain beliefs.
The computation complexity in specific setups was investigated for iterated belief revision \cite{KS_Liberatore1997,KS_SchwindKoniecznyLagniezMarquis2020}, however, investigations on the general computability are missing.

In this paper, we focus on the computational aspect of iterated belief change. %
We quantify its computational power by providing corresponding notions and a construction that demonstrate that iterated belief change is Turing complete, even when the prominent iteration principles of \citeauthor{KS_DarwichePearl1997} (\citeyear{KS_DarwichePearl1997}) are satisfied.
We will therefore consider iterated revision in the framework by \citeauthor{KS_DarwichePearl1997} (\citeyear{KS_DarwichePearl1997}), the currently most prominent framework to iterated belief change.

\section{Formal Background}
\label{sec:background}

Let $ \Sigma=\{a,b,c,\ldots\} $ be a propositional signature (non empty finite set of propositional variables) and $ \propLang $ a propositional language over $ \Sigma $. 
The set of propositional interpretations is denoted by $ \Omega $.
Propositional entailment is denoted by $ \models $, the set of models of $ \alpha $ with $ \modelsOf{\alpha} $, and \( L\subseteq \propLang \) is deductively closed if \( L = \{ \beta\mid L \models \beta \} \).
For a set of worlds $ \Omega'\subseteq \Omega $ and  a total preorder $ {\preceq} \subseteq  \Omega\times  \Omega $ (total, reflexive and transitive relation), we denote with 
$ {\min(\Omega',\preceq)}=\{ \omega\in\Omega' \mid  \omega\preceq \omega' \ksForAll \omega'\in\Omega' \} $
the set of all minimal worlds of $ \Omega' $ with respect to \( \preceq \).
For a total preorder $ \preceq $, we say $ x \prec y $ if $ x \preceq y $ and $ y \not\preceq x $.
With \( \leq \) we denote the usual ordering on the integers and with \( < \)  we denote  the strict part of \( \leq \).

\section{Iterated Belief Revision\\in the Darwiche-Pearl Framework}
AGM theory, by Alchourr{\'{o}}n, G{\"{a}}rdenfors and Makinson (\citeyear{KS_AlchourronGaerdenforsMakinson1985}), 
deals with the dynamics of beliefs in the context of belief sets, i.e., deductively closed sets of
propositions. 

The area of iterated belief change abstracts from belief sets to \emph{epistemic states}, 
in which the agent maintains necessary information for all used belief change operators. 
The most prominent approach for iterated belief change is the framework by \citeauthor{KS_DarwichePearl1997} (\citeyear{KS_DarwichePearl1997}), which does leave open what an epistemic state is, but assumes that for every epistemic state $ \Psi $ we can obtain the set of plausible sentences $ \beliefsOf{\Psi}\subseteq \mathcal{L} $ of $ \Psi $, which is deductively closed.

While the Darwiche-Pearl framework is very successful, 
an ontological problem of belief change itself appears in their framework: when describing change operators over epistemic states, we have to specify on which epistemic states the changes happen, an aspect which was not explicitly treated by Darwiche and Pearl and rediscovered several times \cite{KS_FriedmanHalpern1996,KS_AravanisPeppasWilliams2019,KS_SchwindKonieczny2020}.

We solve this problem by considering an abstract set of epistemic states $ \setAllES $, and assume in this paper that $ \setAllES $ is a countable infinite set and that for every belief set $ L\subseteq \propLang $ there exists infinite many epistemic states $ \Psi\in\setAllES $ with $ \beliefsOf{\Psi}=L $.
We write $ \Psi\models\alpha $ if $ \alpha\in\beliefsOf{\Psi} $ and we define $ \modelsOf{\Psi}=\modelsOf{\beliefsOf{\Psi}}  $.
\begin{definition}
	A  \emph{belief change operator} (for  $ \setAllES $ and $\propLang$) is a function $ \circ : \setAllES \times \propLang \to 
	\setAllES $.
\end{definition}

As we do not specify in detail what an epistemic states in $ \setAllES $ is, one has to identify the additional information for the elements of  $ \Psi $ for realizing certain operators. 
Formally, we do this by assigning additional information (similar as \citeauthor{KS_KatsunoMendelzon1992}, \citeyear{KS_KatsunoMendelzon1992}) to epistemic states.
\begin{definition}
	A function $ \Psi \mapsto {\preceq_\Psi} $ which assigns to each epistemic states $ \Psi \in \setAllES $ a preorder $ \preceq_\Psi $ over $ \Omega $ is called a \emph{faithful assignment} if \( \modelsOf{\Psi}\neq\emptyset \) implies $ \modelsOf{\Psi}=\min(\Omega,\preceq_\Psi) $.
\end{definition}
AGM revision is characterized by selecting minimal models in the respective assigned total preorders.
\begin{proposition}[AGM Revision for Epistemic States {\cite{KS_DarwichePearl1997}}]\label{prop:es_revision}
	A belief change operator $ \revision $ is an \emph{AGM revision operator for epistemic states} if there is a faithful assignment $ \Psi\mapsto {\preceq_\Psi} $, which assigns to $ \Psi  $ a total preorder $ \preceq_{\Psi} $ such that:
	\begin{equation}\label{eq:repr_es_revision}
		\modelsOf{\Psi \revision \alpha} = \min(\modelsOf{\alpha},{\preceq_{\Psi}}) %
	\end{equation}
\end{proposition}	

\noindent  Driven by the insight that iteration needs additional constraints, \citeauthor{KS_DarwichePearl1997} (\citeyear{KS_DarwichePearl1997}) proposed the following postulates:

\begin{itemize}
    \item[]\hspace{-1.1em}{\normalfont(\textlabel{DP1}{pstl:DP1})} \( \ksIF \beta\models\alpha \ksTHEN \beliefsOf{\Psi \revision \alpha \revision\beta} = \beliefsOf{\Psi \revision \beta} \)
	\item[]\hspace{-1.1em}{\normalfont(\textlabel{DP2}{pstl:DP2})} \( \ksIF \beta\models\negOf{\alpha} \ksTHEN \beliefsOf{\Psi \revision \alpha \revision\beta} = \beliefsOf{\Psi \revision \beta} \)
	\item[]\hspace{-1.1em}{\normalfont(\textlabel{DP3}{pstl:DP3})} \( \ksIF \Psi\revision\beta \models \alpha \ksTHEN (\Psi \revision \alpha) \revision\beta \models \alpha \)
	\item[]\hspace{-1.1em}{\normalfont(\textlabel{DP4}{pstl:DP4})} \( \ksIF \Psi\revision\beta \not\models \negOf{\alpha} \ksTHEN (\Psi \revision \alpha) \revision\beta \not\models \negOf{\alpha} \)
\end{itemize}
It is well-known that these operators can be characterised in the semantic framework of total preorders.
\begin{proposition}[{\citeauthor{KS_DarwichePearl1997} (\citeyear{KS_DarwichePearl1997})}]\label{prop:it_es_revision}
	Let $ \revision $ be an AGM revision operator for epistemic states. Then $ \revision $ satisfies \eqref{pstl:DP1}--\eqref{pstl:DP4} if and only there exists a faithful assignment $ \Psi\mapsto{\preceq_{\Psi}} $ such that $ \preceq_{\Psi} $ is a total preorder, Equation \eqref{eq:repr_es_revision} holds, and the following postulates are satisfied:
	\begingroup
	\begin{itemize}
		\item[]\hspace{-1.1em}{\normalfont(\textlabel{CR1}{pstl:RR8})}~\( \ksIF \omega_1,\omega_2 \in \modelsOf{\alpha} \ksTHEN \omega_1 \!\preceq_{\Psi}\! \omega_2 \Leftrightarrow \omega_1 \!\preceq_{\Psi \revision \alpha}\! \omega_2 \) 
		
		\item[]\hspace{-1.1em}{\normalfont(\textlabel{CR2}{pstl:RR9})}~\( \ksIF \omega_1,\omega_2 \in \modelsOf{\negOf{\alpha}} \ksTHEN \omega_1 \!\preceq_{\Psi}\! \omega_2 \Leftrightarrow \omega_1 \!\preceq_{\Psi \revision \alpha}\! \omega_2 \)
		
        \item[]\hspace{-1.1em}{\normalfont(\textlabel{CR3}{pstl:RR10})}~\( \ksIF \omega_1 {\in} \modelsOf{\alpha},\, \omega_2 {\in} \modelsOf{\negOf{\alpha}}  \ksTHEN    \omega_1 \!\prec_{\Psi}\! \omega_2 \! \Rightarrow \! \omega_1 \!\prec_{\Psi \revision \alpha}\! \omega_2 \)
		
        \item[]\hspace{-1.1em}{\normalfont(\textlabel{CR4}{pstl:RR11})}~\( \ksIF \omega_1  \!\in\! \modelsOf{\alpha} ,\, \omega_2 \!\in\! \modelsOf{\negOf{\alpha}}  \ksTHEN    \omega_1 \!\preceq_{\Psi}\! \omega_2 \! \Rightarrow \! \omega_1 \!\preceq_{\Psi \revision \alpha}\! \omega_2 \)
	\end{itemize}
	\endgroup
\end{proposition}%

The postulates \eqref{pstl:DP1}--\eqref{pstl:DP4} 
are accepted today as very importation iteration principles for revision.

\section{Turing Completeness}
\label{sec:turing_completeness}

In this section, we will show that iterated belief revision is inherently computationally powerful, in particular, we will see that belief change operators can simulate Turing machines.

\subsection{Computability Theory}

At first, we recall basic concepts of computability theory and formal language theory.
In the following, $ \Gamma=\{\sigma,\ldots\} $ will denote a finite (non-empty) alphabet and $ \Gamma^*=\{\varepsilon,\sigma,\sigma\sigma,\ldots\} $ denotes the set of all finite words over $ \Gamma $, where $ \cdot^* $ is the Kleene star operator and $ \varepsilon $ denotes the empty word.

A (deterministic) \emph{Turing machine} is a tuple $ T=\tuple{Q,\Gamma,\delta,q_s} $ where $ Q=\{q_1,\ldots,q_n\} $ is a finite set of \emph{states} where $ q_s\in Q $ is the \emph{starting state},
the set $ \Gamma $ is the finite alphabet of tape symbols, and $ \delta: Q\times \Gamma \to Q \times\{\leftarrow,\downarrow,\rightarrow\} $ is the \emph{transition function} of $ T $.
The state $ q_\mathtt{halt}:=q_n $ is used a special state, indicating that the machine $ T $ has stopped.
We describe the \emph{position} on the tape of $ T $ as triple $\tuple{t_l,t_c,t_r} \in \Gamma^*\times \Gamma\cup\{\varepsilon\} \times\Gamma^* $.
A \emph{configuration} of $ T $ is a tuple $ \tuple{q,\tuple{t_l,t_c,t_r}}\in Q\times \Gamma^*\times \Gamma\cup\{\varepsilon\} \times\Gamma^*  $.  
The start configuration for $ T $ by input $ \gamma=\sigma_1\sigma_2\ldots\sigma_n\in\Gamma^* $ is $ (q_s,\tuple{\varepsilon,\sigma_1,\sigma_2\ldots\sigma_n}) $.
The run $ r $ of $ T $ by input $ \gamma\in\Gamma^* $ is defined in the usual way as a (possible infinite) sequence of configurations, starting in the corresponding start configuration by input $ \gamma\in\Gamma^* $, and all subsequent configurations are determined by $ \delta $ in the usual way. 
A configuration $ \tuple{q,\tuple{t_l,t_c,t_r}} $  with $ q=q_\mathtt{halt} $ is a \emph{halting configuration}, and we require that after a halting configuration there are no subsequent configurations in a run. 
If the run $ r $ starting in $ \gamma $ ends in a halting configuration $ \tuple{q,\tuple{t_l,t_c,t_r}} $ we say the word $ t_l t_c t_r $ is the output of $ \gamma $ by $ T $, if $ r $ does not halt, then $ \infty $  is the output of $ \gamma $ by $ T $. 
\begin{definition}
	A partial function\footnote{$ f(\gamma)=\infty $ denotes undefined} $ f:\Gamma^*\to\Gamma^* $  is called \emph{Turing computable} if there is a Turing machine $ T=\tuple{Q,\Gamma,\delta,q_s} $ such that $ f(\gamma) $  is the output of $ \gamma $ by $ T $ for every $ \gamma\in\Gamma^* $.
	If this is the case, we say that \emph{$ T $ computes 
    $ f $}.
\end{definition}

\subsection{Ranking Function-Based Change Operators}
A function \( \kappa : \Omega \to \mathbb{N}_0 \) is a \emph{ranking function} if there exists some \( \omega\in\Omega \) such that \( \kappa(\omega
)=0 \). Let \( \modelsOf{\kappa}=\{ \omega\in\Omega \mid \kappa(\omega)=0 \} \) and \( \beliefsOf{\kappa}=\{ \alpha\in\propLang \mid \modelsOf{\kappa} \subseteq \modelsOf{\alpha} \} \).
Every \( \kappa \) gives rise to a total preorder \( \preceq_{\kappa} \), given by \( \omega_1 \preceq_{\kappa} \omega_2 \) if \( \kappa(\omega_1) \leq \kappa(\omega_2) \).
Ranking function are a common knowledge representation formalism \cite{KS_Spohn2012}.

For our main theorem we make use of special change operators, which can be characterized by information in form of a ranking function.
\begin{definition}\label{def:rankingbased}
    A belief change operator $ \circ $ is called \emph{ranking function-based} if there is an assignment $ \Psi \mapsto (\kappa_{\Psi},b_\Psi) $ such that the following is satisfied:
    \begin{description}
        \item[\normalfont(RFA0)] if \( \modelsOf{\Psi}=\emptyset  \), then \( b_\Psi=\bot \); otherwise \( b_\Psi=\top \)
        \item[\normalfont(RFA1)] if \( b_\Psi=\top  \), then \( \modelsOf{\kappa}=\modelsOf{\Psi} \)
        \item[\normalfont(RFA2)] \( \ksIF \kappa_{\Psi}  = \kappa_{\Phi} \ksAND b_\Psi=b_\Phi  \ksTHEN \Psi \circ \alpha = \Phi \circ \alpha \)
    \end{description}
\end{definition}

Note that Definition \ref{def:rankingbased} makes again no assumption about the real form of epistemic states, but guarantees that the behaviour of a belief change operator is completely representable by an assigned ranking function $ \kappa_{\Psi} $. 
The following proposition relies heavily on our assumptions about \( \setAllES \).
\begin{proposition}\label{prop:rankstuff}
    A belief change operator $ \revision $ is an AGM revision operator for epistemic states if and only if \( \revision \) is ranking function-based with $ \Psi \mapsto (\kappa_{\Psi},b_\Psi) $ such that:
   	\begin{equation}\label{eq:repr_es_revision_rank}
        \modelsOf{\Psi \revision \alpha} = \min(\modelsOf{\alpha},{\preceq_{\kappa_\Psi}}) %
    \end{equation}
\end{proposition}
\begin{proof} We consider both directions of the proof independently, but due to space reasons we focus only on \( \Psi \) with \( \modelsOf{\Psi}\neq\emptyset \), and do not consider the \( b_\Psi \)-part.
    \emph{The \enquote{if} direction.}
    Let $ \Psi\mapsto {\preceq_\Psi} $ be a faithful assignment which satisfies Equation~\eqref{eq:repr_es_revision}, guaranteed by Proposition \ref{prop:es_revision}.
    With \( \mathrm{Behv}(\Psi) = \{  \Phi\in\setAllES \mid {\preceq_\Psi}={\preceq_\Phi} \text{ and } \Psi \revision \alpha = \Phi \revision \alpha  \text{ for all } \alpha\in\propLang \}  \) we denote all \emph{behavioural} equivalent states to \( \Psi \). Clearly, \( \Phi \in \mathrm{Behv}(\Psi) \) if and only if \( \mathrm{Behv}(\Psi)=\mathrm{Behv}(\Phi) \).
     Note that \( \mathrm{Behv}(\Psi) \) is countable and that for each \( \preceq_\Psi \) there exist countably infinite many ranking functions \( \kappa \) with \(  {\preceq_\Psi} = {\preceq_{\kappa}} \).
    Because of these properties, we can construct $ \Psi \mapsto \kappa_{\Psi} $ 
    by choosing for each different set \( \mathrm{Behv}(\Psi) \) a unique ranking function \( \kappa_{\Psi} \) with \(  {\preceq_\Psi} = {\preceq_{\kappa_{\Psi}}} \).
    By construction (RFA1), (RFA2) and Equation~\eqref{eq:repr_es_revision_rank} are satisfied by $ \Psi \mapsto \kappa_{\Psi} $.
    
    \emph{The \enquote{only if} direction.}     Every assignment $ \Psi \mapsto {\kappa_{\Psi}} $ as given in Definition \ref{def:rankingbased} which satisfies Equation~\eqref{eq:repr_es_revision_rank} yields a faithful assignment \( \Psi\mapsto {\preceq_{\kappa_\Psi}} \) which satisfies Equation~\eqref{eq:repr_es_revision}. Thus, by Proposition \ref{prop:es_revision}, existence of $ \Psi \mapsto \kappa_{\Psi} $ as above implies that \( \revision \) is an AGM revision operator for epistemic states.
\end{proof}

\noindent Because of Proposition \ref{prop:rankstuff}, we will focus on ranking function-based change operators and identify epistemic states with ranking functions, 
and use \( \kappa \) as a synonym for \( \Psi \) with \( {\modelsOf{\Psi}\neq\emptyset} \).

\subsection{Simulation Machinery}
\begin{figure}[]\centering
    \begin{tikzpicture}
        
        \foreach \x/\l/\w/\m in {
            9/Rank/$\omega$/Intended Meaning,
            8/\vdots/ /,
            7/$n+3$/$\omega_{\mathtt{pos}}$/$\langle \varepsilon;a;baba \rangle$,
            6/$n+2$/ /-,
            5/$n+1$/ /$\langle aab;b;baba \rangle$,
            4/$n$/$\omega_{\mathtt{halt}}$/$q=q_\mathtt{halt}$,
            3/\vdots/ /\vdots,
            2/2/$\omega_\mathtt{q}$/$q=q_2$,
            1/1/ /$q=q_1$,
            0/0/$\omega_{\mathtt{0}}$/-}
        {
            \node at (0,\x*0.5) {\l};
            \node at (1.5,\x*0.5) {\w};
            \node at (4,\x*0.5) {\m};
        }
        
        \draw [] (-0.5,0.5*0.5) -- (5.5,0.5*0.5) ;
        
        \draw [] (-0.5,4.5*0.5) -- (5.5,4.5*0.5) ;
        
    \end{tikzpicture}
    \caption{Sketch of a ranking function represent a configuration.}\label{fig:rankingfunctionfoo}
\end{figure}

\newcommand{\CONFtype}{\textbf{\texttt{CONF}}}
\newcommand{\PEEKtype}{\textbf{\texttt{PEEK}}}
\newcommand{\TRANStype}{\textbf{\texttt{TRANS}}}
\newcommand{\POSTtype}{\textbf{\texttt{POST}}}

We will simulate one step of a Turing machine by multiple changes. 
The simulation uses four types of ranking functions with special meaning in the simulation of~$ T $:
\begin{itemize}\setlength\itemsep{0pt}
        \item[]\hspace{-1.1em}{\CONFtype} Intended to represent a configuration of $ T $.
    
    \item[]\hspace{-1.1em}{\PEEKtype} Intended to check where $ T $ is in a halting state.
    
    \item[]\hspace{-1.1em}{\TRANStype} Describing the change of the state in a transition (before modifying the tape).
    \item[]\hspace{-1.1em}{\POSTtype} An intermediate state after updating state and tape.
\end{itemize}
In particular, simulation of a transition to a new configuration will take three steps. 
As first step we will compute a {\TRANStype} ranking function where we update the state of $ T $ (while saving the old state). Then we update the tape, obtaining a {\POSTtype} ranking function. Adjusting some ranks yields a {\CONFtype} ranking function representing the new configuration (see Algorithm \ref{algo:simulate} for an overview).

The ranking functions of these types are given over a signature $ \Sigma $ with $ |\Sigma|=2 $. 
Consequently, $ \Omega $ has four elements, which we denote by $ \omega_{\mathtt{0}} , \omega_{\mathtt{halt}}, \omega_{q}, \omega_{\mathtt{pos}}  $.
Each interpretation from \( \Omega \) will serve a special purpose  (see Figure \ref{fig:rankingfunctionfoo}).
We use the rank of $ \omega_{q} $ as indicator for the current state of $ T $.
The interpretation $ \omega_{\mathtt{pos}} $ will encode the tape position of $ T $ using the rank of $ \omega_{\mathtt{pos}} $. 
Therefore, remember, that it is well-known that (by a variation of a Gödel numbering) for a fixed $ \Gamma $ we can encode positions $ \tuple{t_l,t_c,t_r} $ into natural numbers. Let $ \mathtt{enc}: \Gamma^*\times \Gamma \times\Gamma^* \to \mathbb{N} $ be such an injective encoding function. 
We use $ \omega_{\mathtt{halt}} $ to check if the halting state is reached (\PEEKtype) and for saving the prior state of \( T \) (\TRANStype).
The interpretation $ \omega_{\mathtt{0}} $ has rank 0 in most situations;
the purpose is to ensure
that we always have a ranking function.

\begin{algorithm}[t]\DontPrintSemicolon\LinesNumbered
    \KwIn{A word $ \gamma\in\Gamma^* $}
    \KwOut{A word over $ \Gamma $ or $ \infty $ if it runs infinit}
    (\CONFtype) $ \kappa \leftarrow \kappa_\gamma^T   $\tcp*{inital conf.}
    (\PEEKtype) $ \kappa \leftarrow \kappa \circ \varphi_{q,\mathtt{halt}}  $\tcp*{prepare peek}
    \While{$  \omega_{\mathtt{halt}} \notin \modelsOf{\kappa} $\tcp*{check if $ q {\neq} q_\mathtt{halt} $}}{
        (\CONFtype) $ \kappa \leftarrow \kappa \circ \varphi_{\mathbf{0},\mathtt{pos}}  $ \tcp*{obtain conf.}
        (\TRANStype) $ \kappa \leftarrow \kappa\circ \varphi_\mathbf{0}  $ \tcp*{update state}
         (\POSTtype) $ \kappa \leftarrow \kappa\circ \varphi_{\mathbf{0},q,\mathtt{halt}}  $\tcp*{update tape}
        (\CONFtype) $ \kappa \leftarrow \kappa\circ \varphi_{\mathbf{0},q,\mathtt{pos}}  $\tcp*{restore $ \omega_\mathtt{halt} $}
        (\PEEKtype) $ \kappa \leftarrow \kappa \circ \varphi_{q,\mathtt{halt}}  $\tcp*{prepare peek}
    }
    $ \tuple{q,\tuple{t_l,t_c,t_r}} \leftarrow \texttt{conf}(\kappa \circ \varphi_{\mathbf{0},\mathtt{pos}})  $\;
    \Return $ t_lt_ct_r $\;
    \caption{\textsf{Simulate-TM}$ (T,\circ,\gamma) $}\label{algo:simulate}
\end{algorithm}

In the following we will describe the four types of ranking functions in more detail:
\begin{itemize}\setlength\itemsep{0pt}
\item[]\hspace{-1.1em}{\CONFtype.} A {\CONFtype} ranking function $ \kappa $ is given if 
\begin{align*}
     \kappa(\omega_{\mathtt{0}})    & =0 & 1 & \leq \kappa(\omega_{q}) \leq n\\
     \kappa(\omega_{\mathtt{halt}}) & =n & n & < \kappa(\omega_{\mathtt{pos}})
\end{align*}
  and $ \texttt{enc}^{-1}(\kappa(\omega_{\mathtt{pos}})-(n+1)) \neq \emptyset $ holds. 
For such a ranking function let
    \( \texttt{conf}(\kappa)=\tuple{q_{\kappa(\omega_{q})},\texttt{enc}^{-1}(\kappa(\omega_{\mathtt{pos}})-(n+1))} \)
denote the represented configuration of \( T \). Note that we can provide for each possible configuration a ranking function that represents the configuration. 
In  particular, for a given Turing machine $ T $ we denote with $ \kappa_\gamma^T $ a ranking function representing the start configuration for $ T $ by input $ \gamma\in\Gamma^* $.

\item[]\hspace{-1.1em}{\PEEKtype.} 
When performing a revision on a {\CONFtype} ranking function by a formula having $ \omega_{\mathtt{halt}} $ and $ \omega_{q} $ as models, rank 0 of  \(  \omega_{\mathtt{halt}} \)  in the posterior ranking function implies that \( T \) has reached a halting state.
Ranking functions \( \kappa \) which are the result of such a revision are called {\PEEKtype} and obey the following characteristics
\begin{align*}
        \kappa(\omega_{q}) &=0 &        0 &\leq \kappa(\omega_{\mathtt{halt}}) \leq n-1 \\
      \kappa(\omega_{\mathtt{0}})&=1& n+1 & < \kappa(\omega_{\mathtt{pos}}) 
\end{align*}
and  \( \texttt{enc}^{-1}(\kappa(\omega_{\mathtt{pos}})-(n+2)) \neq \emptyset \).
For such a ranking function let
    \( \texttt{conf}(\kappa)=\allowbreak\tuple{q_{n-\kappa(\mathtt{halt})},\allowbreak\texttt{enc}^{-1}(\kappa(\omega_{\mathtt{pos}})-(n+2))} \)
denote the represented configuration.

\item[]\hspace{-1.1em}{\TRANStype.} A {\TRANStype} ranking function $ \kappa $ is given if:
\begin{align*}
    \kappa(\omega_{\mathtt{0}}) &=0 &1 &\leq \kappa(\omega_{q}) \leq n \\
    \kappa(\omega_{\mathtt{pos}}) &> n & 2 & \leq \kappa(\omega_{\mathtt{halt}}) \leq n 
\end{align*}
For such a ranking function $ \kappa $ we define $ \texttt{post-state}(\kappa)=q_{\kappa(\omega_{\mathtt{q}})} $.
Moreover, let
    \( \texttt{conf}(\kappa) {=}\allowbreak \tuple{q_{\kappa(\omega_{\mathtt{q}})-n+\kappa(\omega_{\mathtt{halt}})},\texttt{enc}^{-1}(\kappa(\omega_{\mathtt{pos}}){-}(n{+}1))} \)
denote the represented configuration (before the transition). %
Note that $ \kappa(\omega_{q}) $ encodes the posterior state, but the current (prior) state is reconstructible from $ \kappa(\omega_{q}) $ and the difference between $ n $ and $ \kappa(\omega_{\mathtt{halt}}) $.

\item[]\hspace{-1.1em}{\POSTtype.} Any ranking function $ \kappa $ with
\begin{align*}
    \kappa(\omega_{\mathtt{0}})    & =0 & 1 & \leq \kappa(\omega_{q}) \leq n  \\
    \kappa(\omega_{\mathtt{halt}}) & =n & n & < \kappa(\omega_{\mathtt{pos}})
\end{align*}
and with  $ \texttt{enc}^{-1}(\kappa(\omega_{\mathtt{pos}})-(n+1)) \neq \emptyset $.
Such a ranking function represents the configuration
    \( \texttt{conf}(\kappa)=\tuple{q_{\kappa(\omega_{q})},\texttt{enc}^{-1}(\kappa(\omega_{\mathtt{pos}})-(n+1))} \).
\end{itemize}

Simulation is performed by continuously revising by specific formulas (Algorithm \ref{algo:simulate}).
With $ \varphi_{\mathbf{0},\mathtt{pos}} $ we denote a formula  with $ \modelsOf{\varphi_{\mathbf{0},\mathtt{pos}}}=\{ \omega_{\mathtt{0}}, \omega_{\mathtt{pos}} \} $. This extends analogously to $ \varphi_{q,\mathtt{halt}} $, $ \varphi_\mathbf{0} $, $ \varphi_{\mathbf{0},q,\mathtt{halt}} $ and $ \varphi_{\mathbf{0},q,\mathtt{pos}} $.

\begin{definition}
    Let $ f: \Gamma^*\to\Gamma^* $ be a partial function computed by a Turing machine \( T \).
    A ranking function-based belief change operator $ \circ $ is said to \emph{simulate $ T $} if  \textsf{Simulate-TM}$ (T,\circ,\gamma) $ (Algorithm \ref{algo:simulate}) yields the output of $ \gamma $ by $ T $ for all $ \gamma\in\Gamma^* $.
    If $ \circ $ simulates $ T $, then we say $ \circ $ \emph{computes} $ f $.
\end{definition}
 
\subsection{Main Theorem}
We will now provide a construction for an AGM revision operator $ \revision $ which simulates a given Turing machine and show correctness of the construction.
\begin{theorem}\label{thm:turing_complete}
    Every Turing computable function \( f \) is computable by a ranking function-based AGM revision operator that satisfies the postulates \eqref{pstl:DP1}--\eqref{pstl:DP4}.
\end{theorem}
\begin{proof}
    Let $ T=\tuple{Q,\Gamma,\delta,q_s} $ be a (deterministic) Turing machine computing $ f $.
 We construct a ranking function-based AGM revision operator $ \revision $ and an assignment $ \Psi\mapsto (\kappa_{\Psi},b_\Psi) $.
We start by letting $ \Psi\mapsto (\kappa_{\Psi},b) $ be a bijective function such that for each ranking function \( \kappa \) there exist exactly two epistemic states \( \Psi^\kappa,\Psi^\kappa_\bot\in\setAllES \) with \( \modelsOf{\Psi^\kappa}=\modelsOf{\kappa} \), \( \kappa_{\Psi^\kappa}=\kappa \) and \( b_{\Psi^\kappa}=\top \); and with \( \modelsOf{\Psi^\kappa_\bot}=\emptyset \), \( \kappa_{\Psi^\kappa_\bot}=\kappa \) and \( b_{\Psi^\kappa_\bot}=\bot \).
Such a mapping exists by the properties of \( \setAllES \). In the following we ease notation and just write \( \kappa \) for \( \Psi^\kappa \) and \( \kappa_\bot \) for \( \Psi^\kappa_\bot \).

    Next, we provide how $ \revision $ behaves in all cases appearing in runs of \textsf{Simulate-TM}$ (T,\revision,\gamma) $ (Algorithm \ref{algo:simulate}), depending on the ranking function type of  $ \kappa $:
    \begin{itemize}\setlength\itemsep{0pt}
        \item[]\hspace{-1.1em}\textbf{\texttt{CONF}} 
        We consider two different cases.
        
        These case of $ \alpha=\varphi_{q,\mathtt{halt}} $. For this case let $ \revision $ yield the following ranking function $ \kappa \revision \alpha $:
        \begin{align*}
            \kappa \revision \alpha(\omega_{\mathbf{0}}) & = 1 & \kappa \revision \alpha(\omega_{\mathtt{halt}}) & = \kappa(\omega_{\mathtt{halt}}) - \kappa(\omega_{q}) \\
            \kappa \revision \alpha(\omega_{q}) & = 0 & \kappa \revision \alpha(\omega_{\mathtt{pos}}) & =\kappa(\omega_{\mathtt{pos}}) +1
        \end{align*}
        Note that $ \kappa \revision \alpha $ is a {\PEEKtype} ranking function and $ \texttt{conf}(\kappa)=\texttt{conf}(\kappa \revision \alpha) $.
The case of \( \alpha=\varphi_{\mathbf{0}} \).
        Let $ \kappa_\Psi $ be a ranking function representing a configuration $ C=\tuple{q_i,\tuple{t_l,t_c,t_r}} $ of $ T $.
        Let $ \tuple{q_j,\tuple{t_l',t_c',t_r'}} $ be the successor configuration provided by $ \delta $.
        Then $ \revision $ yields the following ranking function $ \kappa \revision \alpha $:
        \begin{align*}
            \kappa \revision \alpha(\omega_{\mathbf{0}})    & = 0              & \kappa \revision \alpha(\omega_{\mathtt{halt}}) & = \kappa(\omega_{\mathtt{halt}})-(i-j)             \\
            \kappa \revision \alpha(\omega_{q})             & = j & \kappa \revision \alpha(\omega_{\mathtt{pos}})  & = \kappa(\omega_{\mathtt{pos}})-(i-j) 
        \end{align*}
        Note that $ \kappa \revision \alpha $ is a {\TRANStype} ranking function with $ \texttt{post-state}(\kappa \revision \alpha)=q_j $ and $ \texttt{conf}(\kappa)=\texttt{conf}(\kappa \revision \alpha) $.

    \item[]\hspace{-1.1em}{\PEEKtype} Based on Algorithm \ref{algo:simulate} we consider $ \alpha=\varphi_{\mathbf{0},\mathtt{pos}} $.
    Then $ \revision $ is constructed such that it yields $ \kappa \revision \alpha $ by input $ \kappa $ and $ \alpha $ with:
    \begin{align*}
        \kappa \revision \alpha(\omega_{\mathbf{0}})    & = 0       & \kappa \revision \alpha(\omega_{q})             & = n-\kappa(\omega_{\mathtt{halt}})                    \\
        \kappa \revision \alpha(\omega_{\mathtt{halt}}) & = n             & \kappa\revision\alpha(\omega_{\mathtt{pos}}) & =\kappa(\omega_{\mathtt{pos}})   -1            
    \end{align*}
    The ranking function $ \kappa \revision \alpha $ is of {\CONFtype} and $ \texttt{conf}(\kappa)=\texttt{conf}(\kappa \revision \alpha) $.
    
    \item[]\hspace{-1.1em}{\TRANStype} Let $ \texttt{post-state}(\kappa \revision \alpha)=q_j $ and $ \text{conf}(\kappa)=\tuple{q_i,\tuple{t_l,t_c,t_r}}$ and successor configuration $ C=\tuple{q_j,\tuple{t_l',t_c',t_r'}} $.
    For $ \alpha=\varphi_{\mathbf{0},q,\mathtt{halt}} $ let $ \revision $ yield the ranking function $ \kappa \revision \alpha $ given by:
    \begin{align*}
        \kappa \revision \alpha(\omega_{\mathbf{0}})    & = 0                                &\kappa \revision \alpha(\omega_{\mathtt{halt}}) & = \kappa(\omega_{\mathtt{halt}})      \\
        \kappa \revision \alpha(\omega_{q})             & = \kappa(\omega_{q})        & \kappa\revision\alpha(\omega_{\mathtt{pos}})    & = n+\texttt{enc}(C)+1      
    \end{align*}
    Note that $ \kappa \revision \alpha $ is a {\POSTtype} ranking function which represents the posterior configuration $ C $, i.e. $ C=\texttt{conf}(\kappa \revision \alpha) $.
    
    \item[]\hspace{-1.1em}{\POSTtype}     As last case consider $ \alpha=\varphi_{\mathbf{0},q,\mathtt{pos}} $. We define $ \revision $ as follows:
    \begin{align*}
        \kappa \revision \alpha(\omega_{\mathbf{0}})    & = 0                          & \kappa \revision \alpha(\omega_{\mathtt{halt}}) & = n            \\
        \kappa \revision \alpha(\omega_{q})             & = \kappa(\omega_{q})     & \kappa\revision\alpha(\omega_{\mathtt{pos}})    & = \kappa(\omega_{\mathtt{pos}})         
    \end{align*}
    The resulting $ \kappa \revision \alpha $ is a {\PEEKtype} ranking function.
    \end{itemize}

    Consideration of Algorithm \ref{algo:simulate} yields that all operations $ \revision $ performed in Algorithm \ref{algo:simulate} are covered by the cases described here.
    In every step of Algorithm \ref{algo:simulate} every ranking function constructed is of one of the four types and only parameters treated above are applied. 
    In particular, observe that in Lines 6 to 8 a complete transition to a new configuration is performed.
    Clearly, if $ \omega_\mathtt{halt}\in \modelsOf{\kappa} $, where $ \kappa $ is of type \PEEKtype, then $ T  $ has also reached the halting state and $ \kappa $ represents the content of the tape of $ T $ in that configuration.
    Thus, Line 11 of Algorithm \ref{algo:simulate} yields the output of $ T $.

    Next, we complete \( \revision \) to a ranking function-based belief change operator, by setting  \( \kappa\revision\alpha=\kappa^\kappa_\alpha \) for all remaining unspecified cases of \( \revision \), where \( \kappa^\kappa_\alpha \) is given as follows:  
If \( \alpha \) is consistent, let \( \kappa^\kappa_\alpha \) denote the ranking function with:
\begin{align*}
   \kappa^\revision_\alpha(\omega) & =  \begin{cases}
       \kappa(\omega) - \min_{\omega'\in\modelsOf{\alpha}}(\kappa(\omega')) & \text{ if } \omega\in\modelsOf{\alpha}\\
       \kappa(\omega) +1 & \text{ if } \omega\notin\modelsOf{\alpha}\\
   \end{cases}
\end{align*}
If \( \alpha \) is inconsistent, let \( \kappa^\kappa_\alpha = \kappa_\bot \).
A careful examination shows that \( \revision \) satisfies the quantitative postulates:
\begin{itemize}
    \item[]\hspace{-1.1em}{\normalfont(\textlabel{Q1}{pstl:Q1})}~\( \ksIF \omega,\omega' \in \modelsOf{\alpha} \ksTHEN \kappa(\omega){-}\kappa(\omega'){=}\kappa\revision\alpha(\omega){-}\kappa\revision\alpha(\omega') \)
    
    \item[]\hspace{-1.1em}{\normalfont(\textlabel{Q2}{pstl:Q2})}~\( \ksIF \omega,\omega' \in \modelsOf{\negOf{\alpha}} \ksTHEN \kappa(\omega){-}\kappa(\omega'){=}\kappa\revision\alpha(\omega){-}\kappa\revision\alpha(\omega') \)
    
    \item[]\hspace{-1.1em}{\normalfont(\textlabel{Q3}{pstl:Q3})}~\( \ksIF \omega \in \modelsOf{\alpha} \ksAND \omega' \in \modelsOf{\negOf{\alpha}} \),\\\mbox{}\hfill then \( \kappa(\omega)  <  \kappa(\omega')\Rightarrow\kappa\revision\alpha(\omega) < \kappa\revision\alpha(\omega') \)
    
    \item[]\hspace{-1.1em}{\normalfont(\textlabel{Q4}{pstl:Q4})}~\( \ksIF \omega \in \modelsOf{\alpha} \ksAND \omega' \in \modelsOf{\negOf{\alpha}} \),\\\mbox{}\hfill then \( \kappa(\omega) \leq \kappa(\omega')\Rightarrow\kappa\revision\alpha(\omega) \leq \kappa\revision\alpha(\omega') \)
\end{itemize}
    Consequently, by Proposition \ref{prop:rankstuff} and Proposition \ref{prop:it_es_revision}, \( \revision \) is an AGM revision operator that satisfies \eqref{pstl:DP1}--\eqref{pstl:DP4}.~\qedhere
\end{proof}
Theorem \ref{thm:turing_complete} gives rise to the central observation.
\begin{observation}%
    Iterated belief change is Turing complete.
\end{observation}

\section{Conclusion}
\label{sec:conclusion}
In this paper we showed that iterated belief change is Turing complete; this includes AGM revision in the context of iterated belief change.
In particular, we showed how Turing machines can be encoded in a belief revision operator.
The results of Turing-completeness still hold when assuming the Darwiche-Pearl iteration principles for belief revision.

It is reasonable that (in certain situations) iterated belief change should be limited in its computational power, e.g., when modeling humans. In future work, we will identify such situations formally and investigate computational properties and restrictions by means of iteration postulates. 

\clearpage
\bibliographystyle{kr}
\bibliography{bibexport}

\end{document}